\documentclass[11pt]{amsart}
\usepackage{graphicx, amsmath, fullpage, amssymb, amsthm, amsfonts, color, algorithm}
\usepackage{algorithm,tikz}
\usepackage{enumerate}
\newtheorem{theorem}{Theorem}[section]
\newtheorem{lemma}[theorem]{Lemma}

\theoremstyle{definition}
\newtheorem{definition}[theorem]{Definition}

\usepackage[toc,page]{appendix} 

\theoremstyle{remark}

\numberwithin{equation}{section}
\usepackage[colorlinks,
            linkcolor=red,
            anchorcolor=red,
            citecolor=blue
            ]{hyperref}
 
\newcommand{\R}{\mathbb R}
\newcommand{\C}{\mathbb C}

\newcommand{\x}{\mathbf x}
\newcommand{\y}{\mathbf y}
\newcommand{\z}{\mathbf z}
\newcommand{\vv}{\mathbf v}
\newcommand{\CC}{\mathcal C}
\newcommand{\expval}{\mathbb E}
\newcommand{\1}{\mathbf 1}
\newcommand{\frob}{\mathrm F}

\renewcommand{\epsilon}{\varepsilon}

\DeclareMathOperator{\argmax}{\arg\,\max}
\DeclareMathOperator{\prb}{\mathbb P}

\begin{document}

\title{Exact Recovery in the Hypergraph Stochastic Block Model: A Spectral Algorithm}
\author{Sam Cole}
\address{Department of Mathematics, University of Manitoba, Winnipeg, MB R3T 2MB, Canada}
\email{samuel.cole@umanitoba.ca}

\author{Yizhe Zhu}
\address{Department of Mathematics, University of California, San Diego, La Jolla, CA 92093, USA}
\email{yiz084@ucsd.edu}
\thanks{}

\date{\today}
%
\date{\today}
%
%

\begin{abstract}
We consider the exact recovery problem in the hypergraph stochastic block model (HSBM) with $k$ blocks of equal size.  More precisely, we consider a random $d$-uniform hypergraph $H$ with $n$ vertices partitioned into $k$ clusters of size $s = n / k$.  Hyperedges $e$ are added independently with probability $p$ if $e$ is contained within a single cluster and $q$ otherwise, where $0 \leq q < p \leq 1$. We present a spectral algorithm which  recovers the clusters exactly with high probability, given mild conditions on $n, k, p, q$, and $d$.  Our algorithm is based on the \emph{adjacency matrix} of $H$, which is a symmetric $n \times n$ matrix whose $(u, v)$-th entry is the number of hyperedges containing both $u$ and $v$.  To the best of our knowledge, our algorithm is the first to guarantee exact recovery when the number of clusters $k=\Theta(\sqrt{n})$. 
\end{abstract}

\maketitle



\section{Introduction}

\subsection{Hypergraph clustering}
Clustering is an important topic in data mining, network analysis, machine learning and computer vision. Many clustering methods are based on graphs, which represent  pairwise relationships among objects. However, in many real-world problems, pairwise relations are not sufficient, while  higher order relations between objects  cannot be represented as edges on graphs. Hypergraphs can be used to represent more complex relationships among data, and they have been shown empirically to have advantages over graphs; see \cite{zhou2007learning,gonzalez2007hypergraph}. Thus, it is of practical interest to develop algorithms based on hypergraphs that can handle higher-order relationships among data, and much work has already been done to that end; see, for example, \cite{zhou2007learning,leordeanu2012efficient,vazquez2009finding,ghoshdastidar2015spectral,bulo2009game,hein2013total,alistarh2015streaming}. Hypergraph clustering  has found a wide range of applications (\cite{han1998hypergraph,ducournau2012reductive,bretto2005hypergraph,gallagher2013clustering, kim2011higher}).  

The stochastic block model (SBM) is a generative model for random graphs with community structures which serves as a useful benchmark for the task of recovering community structure from graph data. It is natural to have an analogous model for random hypergraphs as a testing ground for hypergraph clustering algorithms. 

\subsection{Hypergraph stochastic block models}\label{sec:hsbmintro}
The \textit{hypergraph stochastic block model}, first introduced in \cite{ghoshdastidar2014consistency}, is a generalization of the SBM for  hypergraphs.  
We define the hypergraph stochastic block model (HSBM) as follows for $d$-uniform hypergraphs.
\begin{definition}[Hypergraph]
	A $d$-uniform hypergraph $H$ is a pair $H=(V,E)$ where $V$ is a set of vertices and $E\subset {V \choose d}$ is a set of subsets with size $d$ of $V$, called hyperedges.  when $d=2$, it is the same as an ordinary graph.
\end{definition}

\begin{definition}[Hypergraph stochastic block model (HSBM)]
Let $\mathcal C=\{C_1,\dots C_k\}$ be a partition of the set $[n]$ into $k$ sets of size $s=n/k$ (assume $n$ is divisible by $k$), each $C_i, 1\leq i\leq k$ is called a cluster. For constants $0\leq q<p<1$, we define the $d$-uniform hypergraph SBM as follows:

For any set of $d$ distinct vertices $i_1,\dots i_d$, generate a hyperedge $\{i_1,\dots i_d\}$ with probability $p$ if the vertices $i_1,\dots i_d$ are in the same cluster in $\mathcal C$. Otherwise, generate the hyperedge $\{i_1,\dots i_d\}$ with probability $q$.
We denote this distribution of random hypergraphs as $H(n,d,\mathcal C,p,q)$.	  When $d=2$, it is the same as the stochastic block models for random graphs.
\end{definition}

Hypergraphs are closely related to symmetric tensors. We give a definition of symmetric tensors below.  See, e.g., \cite{kolda2009tensor}, for more details on tensors.
\begin{definition}[Symmetric tensor]\label{def:tensor}
	Let $T\in\mathbb R^{n\times \cdots \times n}$ be an order-$d$ tensor. We call $T$ symmetric if $T_{i_1,i_2,\dots i_d}=T_{\sigma(i_1),\sigma({i_2})\dots, \sigma(i_d)}$ for any $i_1,\dots, i_d\in [n]$ and any permutation $\sigma$ in the symmetric group of order $d$.
\end{definition}

Formally, we can use a random symmetric tensor to represent a random hypergraph $H$ drawn from this model. We construct an \textit{adjacency tensor} $T$ of $H$ as follows. For any distinct vertices $i_1<i_2<\dots <i_d$ that are in the same cluster,
\begin{align*}
	T_{i_1,\dots,i_d}=\begin{cases}
		1 &  \text{ with probability } p,\\
		0 & \text{ with probability } 1-p.
	\end{cases}
\end{align*}
For any distinct vertices $i_1 < \dots < i_d$, if any two of them are not in the same cluster, we have 
\begin{align*}
	T_{i_1,\dots,i_d}=\begin{cases}
		1 &  \text{ with probability } q,\\
		0 & \text{ with probability } 1-q.
	\end{cases}
\end{align*}
We set $T_{i_1,\dots i_d}=0$ if any two of the indices in $\{i_1,\dots i_d\}$ coincide, and we set $T_{\sigma(i_1),\sigma({i_2})\dots, \sigma(i_d)} = T_{i_1,\dots i_d}$ for any permutation $\sigma$.  Furthermore, we may abuse notation and write $T_e$ in place of $T_{i_1, \ldots, i_d}$, where $e = \{i_1, \ldots, i_d\}$.

The HSBM recovery problem is to find the ground truth clusters $\mathcal C=\{C_1,\dots, C_k\}$ either approximately or exactly, given a sample hypergraph from $H(n,d,\mathcal C,p,q)$. We may ask the following questions about the quality of the solutions; see \cite{abbe2018community} for further details  in the graph case:
\begin{enumerate}
	\item \textbf{Exact recovery (strong consistency):} Find $\mathcal C$ exactly (up to a permutation) with probability $1-o(1)$.
	\item \textbf{Almost exact recovery (weak consistency):} Find a partition $\hat{\mathcal C}$ such that $o(1)$ portion of the vertices are mislabeled.
	\item \textbf{Detection:} Find a partition $\hat{\mathcal C}$ which is correlated with the true partition $\mathcal C$.
\end{enumerate}
We are typically interested in one of two regimes:
\begin{itemize}
\item	\textbf{The dense regime}.  In this regime $p$ and $q$ are constant, and the number of clusters $k$ is allowed to grow with $n$.  We then ask: how small can we make the cluster size $s = n/k$ while still being able to guarantee recovery?
\item	\textbf{The sparse regime}.  In this regime $k$ is constant, $s = \Theta(n)$, and $p, q = o(1)$.  We then ask: how small can we make $p$ and $q$ while still being able to guarantee recovery?
\end{itemize}

Several methods have been considered for exact recovery of HSBMs.  In \cite{ghoshdastidar2014consistency}, the authors used spectral clustering based on the hypergraph's Laplacian to recover HSBMs that are dense and uniform. Subsequently, they extended their results to sparse, non-uniform hypergraphs  in \cite{ghoshdastidar2015provable,ghoshdastidar2015spectral,ghoshdastidar2017consistency}. Spectral methods along with local refinements were considered in \cite{chien2018community,ahn2016community}. A semidefinite programing approach was analyzed  in \cite{kim2018stochastic}.

For different sparsity regimes, the efficient algorithms are not the same and there is no `universal' algorithm that works optimally for all different sparsity regimes and  different problems. For example, the detection problems of SBMs with bounded expected degrees are analyzed by  algorithms based on self-avoiding walks  \cite{massoulie2014community} or non-backtracking walks \cite{mossel2018proof,bordenave2018nonbacktracking}. However, for the exact recovery problems in the logarithmic degree regime, the algorithms that achieve the information theoretical threshold are based on semidefinite programming \cite{abbe2016exact} or spectral clustering based on eigenvectors of the adjacency matrices \cite{abbe2017entrywise}.  In this paper we will only focus on the exact recovery problem and  our algorithm might not work well for the almost exact recovery or detection problems.

\subsection{Our results}

In this paper, we present a spectral algorithm for exact recovery which compares well with previously known algorithms in the dense regime.  Our main result is the following:
\begin{theorem}\label{main} Let $p,q,d$ be constant. For sufficiently large $n$, there exists a deterministic, polynomial time algorithm which exactly recovers $d$-uniform HSBMs with probability $1 - \exp(-\Omega(\sqrt{n}))$ if $s= \Omega(\sqrt n)$.\end{theorem}
See Theorem~\ref{mainthm} below for the precise statement. 
Our algorithm is based on the \emph{iterative projection} algorithm developed in~\cite{cole2017simple,cole2018recovering} for the graph case.  We apply this approach to the \emph{adjacency matrix} $A$ of the random hypergraph $H$ (see Definition \ref{def:adjmatrix}). The challenge  is that the adjacency matrix constructed from the adjacency tensor used in the algorithm does not have independent entries. In the process, we prove a non-asymptotic concentration result for the spectral norm of $A$, which may be of independent interest (Theorem~\ref{concentration}) for other random hypergraph problems.

\subsection{Why dense HSBMs?}

While sparse (H)SBMs typically have more applications in data science, the dense case is nonetheless of theoretical importance.  The SBM recovery problem is known alternately as the \emph{planted partition} problem and can be seen as a variant of the \emph{planted clique} probelem originally posed in~\cite{jerrum1992large}.  In the latter, one generates an Erd\H{o}s-R\'enyi random graph $G\left(n, \frac12\right)$ and adds edges deterministically to form a clique on an arbitrary subset of the vertices; the goal is then to determine exactly which vertices were members of the ``planted'' clique w.h.p.  If the planted clique is too small, then there is no way to distinguish it from a randomly occurring clique in $G(n, \frac12)$; thus, the central question is how big the clique must be in order to guarantee (efficient) recovery.

For both planted clique and planted partition, it is statistically impossible to recover the clique or partition w.h.p.\ if the size of the clique or parts of the partition is $O(\log n)$~\cite{chen2016statistical}.  On the other hand, the best known polynomial time algorithms in both problems require the size to be $\Omega(\sqrt n)$ in order to guarantee recovery w.h.p.~\cite{alon1998finding,ames2014guaranteed,chen2014improved,cole2017simple,oymak2011finding}, and there is evidence that this is best that can be done efficiently for exact recovery~\cite{feige2000finding,feldman2017statistical,jerrum1992large}.  If the size of the partition or clique is $s=\Omega(\sqrt {n\log n})$, a simple counting algorithm similar to the one we present in  Appendix \ref{sec:counting} would work \cite{kuvcera1995expected,chen2016statistical}.  However, when the cluster sizes are only required to be $\Theta(\sqrt n)$, the
problem, becomes more difficult because a simple counting argument no longer suffices. It's still a major open problem  in theoretical  computer science to find  polynomial algorithms which succeed w.h.p. in the regime where the size of the clique or the partition is $o(\sqrt n)$.
See Section 1.4.1 in \cite{chen2016statistical} for more discussion. 

 Thus, Theorem~\ref{main} is consistent with the state of the art for planted problems on graphs; moreover, our algorithm for HSBM recovery compares favorably with other known algorithms in the dense case with $k = \omega(1)$ clusters (see Section~\ref{sec:cmp}).  To the best of our knowledge, our algorithm is the first to guarantee exact recovery when all clusters are size $\Theta(\sqrt n)$.  In Section~\ref{sec:lowerbounds}, we include  a more thorough discussion of limitations of SBM recovery algorithms.

 As discussed at the end of Section \ref{sec:hsbmintro}, one should not expect one algorithm that works optimally for all sparsity regimes.
While we focus on the dense case, our algorithm can be adapted to the sparse case as well (See Appendix~\ref{sec:sparse}); however, it does not perform as well as previously known algorithms in~ \cite{kim2018stochastic,lin2017fundamental,chien2018minimax,chien2018community,ghoshdastidar2017consistency} in the sparse regime.  In fact, it is even outperformed by a simple hyperedge counting algorithm, which we present in Appendix~\ref{sec:counting}. The main obstacle is the lack of concentration for sparse hypergraphs: in our spectral algorithm (Algorithm \ref{alg}), to make sure the iterative procedure succeed at every step, one needs to take a  union bound over exponentially many events, which requires a concentration bound of the spectral norm with high enough probability, but sparse random matrices do not concentrate as well as dense random matrices. Optimizing our algorithm for sparse HSBMs is a possible direction for future work.

\subsection{Comparison with previous results}\label{sec:cmp}
 
We compare our spectral algorithm, as well as the simple counting algorithm presented in Appendix~\ref{sec:counting}, with previous exact recovery algorithms, with $p,q,d$ being constant. In \cite{ahn2016community,chien2018community} the regime where $k$ grows with $n$ is not explicitly discussed, so we only include $k=O(1)$ case.
\begin{table}[H]
\centering 
\small  
\begin{tabular}{l l l l } 
\hline\hline 
Paper  & Number of clusters   &Algorithm type  \\ [1ex] 
\hline 
\cite{kim2018stochastic} & $O(1)$  & Semidefinite programming \\[1ex]
\cite{ahn2018hypergraph} & $O(1)$  & Spectral + local refinement \\[1ex]
\cite{chien2018community} & $O(1)$  & Spectral + local refinement \\[1ex]
\cite{ghoshdastidar2017consistency}, Corollary 5.1& $o(\log^{\frac{-1}{2d}}(n)n^{\frac{d-4}{2d}})$ &Spectral + $k$-means  \\[1ex] 
Our result (Algorithm~\ref{alg:counting}) & $O(\log^{\frac{-1}{2d-4}}(n) n^{0.5})$ & Simple counting \\[1ex]
Our result  (Algorithm~\ref{alg}) & $O(n^{0.5})$  &Spectral \\ 
\hline 
\end{tabular}
\end{table}


\section{Spectral algorithm and main results}

Our main result is that Algorithm~\ref{alg} below recovers HSBMs with high probability, given certain conditions on $n, k, p, q,$ and $d$.  It is an adaptation of the \emph{iterated projection} algorithm for the graph case introduced by~\cite{cole2017simple,cole2018recovering}.  The algorithm can be broken down into three main parts:
\begin{enumerate}
\item	Construct an ``approximate cluster'' using spectral methods (Steps~\ref{step:adjmatrix}-\ref{step:approxcluster})
\item	Recover the cluster exactly from the approximate cluster by counting hyperedges (Steps~\ref{step:NvW}-\ref{step:exactrec})
\item	Delete the recovered cluster and recurse on the remaining vertices (Step~\ref{step:delandrec}).
\end{enumerate}

\begin{algorithm}[h]
\caption{}\label{alg}
Given $H = (V, E)$, $|V| = n$, number of clusters $k$, and cluster size $s = n / k$:
\begin{enumerate}
\item	Let $A$ be the adjacency matrix of $H$ (as defined in Section~\ref{sec:weightedgraphs}).\label{step:adjmatrix}
\item	Let $P_k(A) = (P_{uv})_{u, v \in V}$ be the dominant rank-$k$ projector of $A$ (as defined in Section~\ref{sec:projectors}).\label{step:proj}
\item	For each column $v$ of $P_k( A)$, let $P_{u_1,v} \geq \ldots \geq P_{u_{n - 1},v}$ be the entries other than $P_{vv}$ in non-increasing order.  Let $W_v := \{v, u_1, \ldots, u_{s - 1}\}$, i.e., the indices of the $s - 1$ greatest entries of column $v$ of $P_k( A)$, along with $v$ itself.\label{step:Wv}
\item	Let $W = W_{v^*}$, where $v^* := \argmax_v ||P_k(A)\1_{W_v}||_2$, i.e.\ the column $v$ with maximum $||P_k(A)\1_{W_v}||_2$.  It will be shown that $W$ has small symmetric difference with some cluster $C_i$ with high probability (Section~\ref{sec:approxcluster}).\label{step:approxcluster}
\item	For all $v \in V$, let $N_{v, W}$ be the number of hyperedges $e$ such that $v \in e$ and $e \setminus \{v\} \subseteq W$, i.e., the number of hyperedges containing $v$ and $d - 1$ vertices from $W$.\label{step:NvW}
\item	Let $C$ be the $s$ vertices $v$ with highest $N_{v, W}$.  It will be shown that $C = C_i$ with high probability (Section~\ref{sec:exactrec}).\label{step:exactrec}
\item	Delete $C$ from $H$ and repeat on the remaining sub-hypergraph.  Stop when there are $< s$ vertices left.\label{step:delandrec}
\end{enumerate}
\end{algorithm}

\begin{theorem}\label{mainthm}
Let $H$ be sampled from $H(n, d, \CC, p, q)$, where $p$ and $q$ are constant, $\CC = \{C_1, \ldots, C_k\}$ and $|C_i| = s = n / k$ for $i = 1, \ldots, k$.  If $d=o(s)$ and 
\begin{equation}\label{eq:epsilonconds}
  \frac{6d\sqrt{d{n \choose d - 1}}}{{s - 2 \choose d - 2}(p - q)s - 12d\sqrt{d{n \choose d - 1}}}\leq \epsilon\leq \frac{p - q}{32d},
\end{equation}
then for sufficiently large $n$, Algorithm~\ref{alg} exactly recovers $\CC$ with probability $\geq 1 - 2^k \cdot \exp(-s) - nk \cdot \exp\left(-\varepsilon^2{s - 1 \choose d - 1}\right)$.
\end{theorem}

In the theorem above, the size $d$ of the hyperedges is allowed to grow with $n$.  The special case in which $d$ is constant follows easily:

\begin{theorem}
Let $H$ be sampled from $H(n, d, \CC, p, q)$, where $p$ and $q$, and $d$ are constant, $\CC = \{C_1, \ldots, C_k\}$ and $|C_i| = s = n / k$ for $i = 1, \ldots, k$.  If 
$$ s \geq \frac{c_0\sqrt{nd}}{(p - q)^{\frac2{d - 1}}},$$
then Algorithm~\ref{alg} recovers $\CC$ w.h.p., where $c_0$ is an absolute constant.	
\end{theorem}

\begin{proof}
Observe that if
\begin{equation}\label{eq:strongercond}
\frac{18d\sqrt{d{n \choose d - 1}}}{{s - 2 \choose d - 2}(p - q)s} \leq \frac{p - q}{32d},
\end{equation}
then we have
\[\frac{6d\sqrt{d{n \choose d - 1}}}{{s - 2 \choose d - 2}(p - q)s - 12d\sqrt{d{n \choose d - 1}}} \leq \frac{18d\sqrt{d{n \choose d - 1}}}{{s - 2 \choose d - 2}(p - q)s} \leq \frac{p - q}{32d}.\]
Hence, if~\eqref{eq:strongercond} is satisfied, then it is possible to choose $\epsilon$ satisfying~\eqref{eq:epsilonconds}, so Theorem~\ref{mainthm} guarantees that we can recover $\CC$ w.h.p. in this case.  Recall that for nonnegative integers $a \geq b$ we can bound the binomial coefficient ${a \choose b}$ by $\left(\frac ab\right)^b \leq {a \choose b} \leq \left(\frac{ae}b\right)^b$.  The conclusion follows by applying these bounds in~\eqref{eq:strongercond} and solving for $s$. Note that we want the failure probability in Theorem~\ref{mainthm} to be $o(1)$, so we require \[\exp\left(-\epsilon^2{s - 1 \choose d - 1}\right) = o((nk)^{-1}).\]  It is easy to verify that this is satisfied if $d$ is constant.
\end{proof}

In Appendix~\ref{sec:counting} we present a trivial hyperedge counting algorithm.  Algorithm~\ref{alg} beats this algorithm by a factor of $(\log n)^{\frac1{2d - 4}}$.   See Section~\ref{sec:cmp} for comparison with other known algorithms.

The remainder of this paper is devoted to proving Theorem~\ref{mainthm}.  Sections~\ref{sec:weightedgraphs}-\ref{sec:projectors} introduce the linear algebra tools necessary for the proof; Section~\ref{sec:approxcluster} shows that Step~\ref{step:approxcluster} with high probability produces a  set with small symmetric difference with one of the clusters; Section~\ref{sec:exactrec} proves that Step~\ref{step:exactrec} with high probability recovers one of the clusters exactly; and Section~\ref{sec:delandrec} proves inductively that the algorithm with high probability recovers all clusters.  

\subsection{Running time}

In contrast to the graph case, in which the most expensive step is constructing the projection operator $P_k(A)$ (which can be done in $O(n^2k)$ time via truncated SVD~\cite{golub1996matrix,gu2015subspace}), for $d \geq 3$ the running time of Algorithm~\ref{alg} is dominated by constructing the adjacency matrix $A$, which takes $O(n^d)$ time (the same amount of time it takes to simply read the input hypergraph).  Thus, the overall running time of Algorithm~\ref{alg} is $O(kn^d)$.

\section{Reduction to random matrices}\label{sec:weightedgraphs}
Since we do not have  many linear algebra and probability tools for random tensors, it would be convenient if we could work with matrices instead of tensors.
 We propose to analyze the  following adjacency matrix of a hypergraph, originally defined in \cite{feng1996spectra}. 
 
\begin{definition}[Adjacency matrix]\label{def:adjmatrix}
Let $H$ be a random hypergraph generated from $H(n,d,\mathcal C,p,q)$ and let $T$ be the adjacency tensor of $H$. For any hyperedge $e=\{i_1,\dots, i_d\}$, let $T_e$ be the entry in $T$ corresponding to $T_{i_1,\dots, i_d}$. 
We define the adjacency matrix $A$ of $H$ by
\begin{align}\label{adj}
A_{ij}:=\sum_{e:\{i,j\}\in e} T_e.
\end{align}
Thus, $A_{ij}$ is the number of hyperedges in $H$ that contains vertices $i,j$. Note that in the summation \eqref{adj}, each hyperedge is counted once.
\end{definition}
From our definition, $A$ is symmetric, and $A_{ii}=0$ for $1\leq i\leq n$. However, the entries in $A$ are not independent.  This presents some difficulty, but we can still get information about the clusters from this adjacency matrix $A$.

\section{Eigenvalues  and concentration of spectral norms}\label{sec:concentration}
It is easy to see that for $d\geq 2$,
\begin{align*}
\mathbb EA_{ij}=	
\left\{
\begin{aligned}
	&{{s-2}\choose{d-2}}(p-q)+{{n-2}\choose{d-2}}q ,& \text{if $i\not=j$ are in the same cluster},\\
	&{{n-2}\choose{d-2}}q, &\text{if $i,j$ are in different clusters.}
	\end{aligned}
\right.\end{align*}

Let $$ \tilde{A}:=\mathbb EA+\left({{s-2}\choose{d-2}}(p-q)+{{n-2}\choose{d-2}}q\right)I,$$ then $\tilde{A}$ is a symmetric matrix of rank $k$. The eigenvalues for $\tilde{A}$ are easy to compute, hence by a shifting, we have the following eigenvalues for $\mathbb EA$.  Note that we are using the convention $\lambda_1(X) \geq \ldots \geq \lambda_n(X)$ for a $n \times n$ self-adjoint matrix $X$.

\begin{lemma}\label{lemma:EAeigs}
The eigenvalues of ~$\mathbb EA$ are	
\begin{align*}\lambda_1(\mathbb EA)&={{s-2}\choose{d-2}}(p-q)(s-1)+{{n-2}\choose{d-2}}q(n-1),\\
\lambda_i(\mathbb EA)&={{s-2}\choose{d-2}}(p-q)(s-1)-{{n-2}\choose{d-2}}q, \quad 2\leq i\leq k,\\
\lambda_i(\mathbb EA)&=-{{s-2}\choose{d-2}}(p-q)-{{n-2}\choose{d-2}}q, \quad k+1\leq i\leq n.
\end{align*}
\end{lemma}

%
%
%
We can use an $\varepsilon$-net chaining argument to prove a concentration inequality for the spectral norm of $A-\mathbb EA$.

\begin{definition}[$\varepsilon$-net]
An \textit{$\varepsilon$-net} for a compact metric space $(\mathcal X, d)$ is a finite subset $\mathcal N$ of $\mathcal X$ such that for each point $x\in\mathcal X$, there is a point $y\in\mathcal N$ with $d(x,y)\leq \varepsilon$.
\end{definition}

\begin{theorem}\label{concentration}
Let $\|\cdot \|_2$ be the spectral norm of a matrix, we have
\begin{equation}\label{eq:concentration}
\|A-\mathbb EA\|_2\leq 6d\sqrt{d{{n}\choose{d-1}}}
\end{equation}
with probability at least $1-e^{-n}$.
\end{theorem}

When $d=2$, Theorem \ref{concentration} is a concentration result for Wigner matrices. Our result includes the case where $d$ is growing with $n$. Lemma 2 in \cite{ahn2018hypergraph} is a similar concentration result for the adjacency matrix $A$ for random hypergraphs, but with probability $1-O(1/n)$.  However, to make our Algorithm \ref{alg} succeed with high probability with $k=\Theta(\sqrt n)$ many clusters, we need a concentration bound of the spectral norm of $A$ with exponentially small failure probability since we need to take a union bound over $2^k$ many events in Section \ref{sec:delete}. 

\begin{proof}[Proof of Theorem \ref{concentration}]
Consider the centered matrix $M:=A-\mathbb EA$, then each entry $M_{ij}$  is a centered random variable. Let $M_e=T_e-\mathbb E[T_e]$. Let $\mathbb S^{n-1}$ be the unit sphere in $\mathbb R^n$ (using the $l_2$-norm).  By the definition of the spectral norm,
\begin{align*}
\|M\|_2=\sup_{\x\in \mathbb S^{n-1}}|\langle M\x,\x\rangle|.
\end{align*}
 Let $\mathcal N$ be an $\varepsilon$-net on $\mathbb S^{n-1}$.  Then for any $\x\in\mathbb S^{n-1}$, there exists some $\y\in\mathcal N$ such that $\|\x-\y\|_2\leq \epsilon$. Then we have
 $$\|M\x\|_2-\|M\y\|_2\leq \|M\x-M\y\|_2\leq \|M\|_2\|\x-\y\|_2\leq \epsilon \|M\|_2.
 $$
 For any $\y \in \mathcal N$, if we take the supremum over $\x$, we have
 $$(1-\epsilon)\|M\|_2\leq \|M\y\|_2\leq \sup_{\z \in\mathcal N}\|M\z\|_2.
 $$
 Therefore
 \begin{align}\label{eps}
 \|M\|_2\leq \frac{1}{1-\varepsilon}\sup_{\x\in\mathcal N}\|M\x\|_2= \frac{1}{1-\varepsilon}\sup_{\x\in\mathcal N}\langle M\x, \x \rangle	
 \end{align}

Now we fix an $\x = (x_1, \ldots, x_n)^\top \in\mathbb S^{n-1}$ first, and   prove a concentration inequality for $\|M\x\|_2$.
Let $E$ be the hyperedge set in a complete $d$-uniform hypergraph on $[n]$. We have
\begin{align*}
\langle M\x,\x\rangle &=\sum_{i\not=j}M_{ij}x_ix_j
                    =2\sum_{i<j}M_{ij}x_ix_j
                    =2\sum_{i<j}(\sum_{e\in E: i,j\in e}M_e)x_ix_j
                =2\sum_{e\in E}(\sum_{i,j\in e, i<j}x_ix_j)M_e.
\end{align*}
Let $Y_e:=(\sum_{ij\in e, i<j}x_ix_j)M_e$, then $$\langle M\x,\x\rangle=2\sum_{e\in E} Y_e$$ where  $\{Y_e\}_{e\in E}$ are  independent. Note that $|M_e|\leq 1$, so we have
$$|Y_e|=|\sum_{ij\in e, i<j}x_ix_jM_e|\leq |\sum_{ij\in e, i<j}x_ix_j|
$$ 
By Hoeffding's inequality,
\begin{align}\label{hoeff}
\mathbb P(|\sum_{e\in E}Y_e|\geq t)\leq 2\exp\left(-\frac{2t^2}{4\sum_{e\in E}|\sum_{ij\in e, i<j}x_ix_j|^2}\right).
\end{align}
From Cauchy's inequality, we have 
\begin{align}
\sum_{e\in E}|\sum_{ij\in e, i<j}x_ix_j|^2&\leq {{d}\choose{2}}\sum_{e\in E}\sum_{ij\in e,i<j}x_i^2x_j^2\leq {{d}\choose{2}}{{n-2}\choose{d-2}}\sum_{1\leq i<j\leq n}x_i^2x_j^2\notag \\
&\leq {{d}\choose{2}}{{n-2}\choose{d-2}}\frac{1}{2}\left(\sum_{i}x_i^2\right)^2\leq \frac{1}{4}d^2{{n}\choose{d-2}}.\label{cauchy}
\end{align}
Therefore from \eqref{hoeff} and \eqref{cauchy},
\begin{align*}
\mathbb P(|\langle M\x,\x\rangle|\geq 2t)\leq 2\exp\left(-\frac{2t^2}{{{n}\choose{d-2}}d^2}\right)	.
\end{align*}
Taking $\displaystyle t=\frac{3}{2}d\sqrt{d{{n}\choose{d-1}}}$, we have
\begin{align*}
\mathbb P\left(|\langle M\x,\x\rangle|\geq 3d\sqrt{d}\sqrt{{{n}\choose{d-1}}}\right)\leq 2\exp\left(-\frac{9d{{n}\choose{d-1}}}{2{{n}\choose{d-2}}}\right)\leq \exp(-3n).	
\end{align*}
Since $|\mathcal N |\leq (\frac{2}{\varepsilon}+1)^n$ (see Corollary 4.2.11 in \cite{vershynin2018high} for example), we can take $\varepsilon=1/2$ and by a union bound,  we have
\begin{align}\label{prob}
\mathbb P\left(\sup_{\x\in\mathcal N}|\langle M\x,\x\rangle|\geq 3d\sqrt{d}\sqrt{{{n}\choose{d-1}}}\right)\leq 5^n\exp(-3n)\leq e^{-n}.	
\end{align}
So we have from \eqref{eps}, \eqref{prob} 
\begin{align*}
\mathbb P\left(\|M\|_2\geq 6d\sqrt{d{{n}\choose{d-1}}}\right)\leq e^{-n}.
\end{align*}
\end{proof}

%

Since $|\lambda_i(A) - \lambda_i(\expval A)| \leq \|A - \expval A\|_2$ for $1\leq i\leq n$, we see that the largest $k$ eigenvalues of $A$ are separated from the remaining $n - k$ by at least ${s - 2 \choose d - 2}(p - q)s - 12d\sqrt{d{n \choose d - 1}}$.  Figure~\ref{fig:eigs} depicts this separation in the eigenvalues, which is necessary to bound the difference in the dominant rank-$k$ projectors of $A$ and $\expval A$ in the next section.

\begin{figure}\centering
\begin{tikzpicture}[scale=.08]
\draw[<->] (-25, 0) -- (125, 0);
\foreach \x/\lbl in {15/{$8\sqrt n$}, 35/{}}
\draw (\x, 3) -- (\x, -3);
\foreach \x in {-5, 40, 115}
\fill (\x, 0) circle(1);
\foreach \x in {-9.8, 5, 10, 11.8, -19.2, -17, -16, -12.2, 0, -1.8, 2.4, 12.8, 36, 45, 47, 49, 53}
\fill (\x, 0) circle(1);
\draw (115, 2) node[anchor=south]{$\lambda_1$};
\draw (45, 2) node[anchor=south]{$\lambda_2, \ldots, \lambda_k$};
\draw (0, 2) node[anchor=south]{$\lambda_{k + 1}, \ldots, \lambda_n$};
\foreach \x in {15, 35}
\draw (\x, -5) -- (\x, -7);
\draw (15, -6) -- (35, -6);
\draw (25, -6)  node[anchor=north] {${s - 2 \choose d - 2}(p - q)s - 12d\sqrt{d{n \choose d - 1}}$};
\end{tikzpicture}
\caption{The distribution of eigenvalues of $A$.}\label{fig:eigs}
\end{figure}
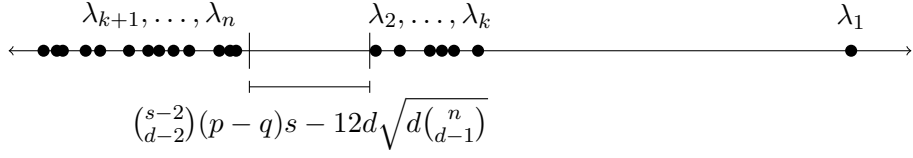

\section{Dominant eigenspaces and projectors}\label{sec:projectors}

Our recovery algorithm is based on the \emph{dominant rank-$k$ projector} of the adjacency matrix $A$.

\begin{definition}[Dominant eigenspace]
If $X$ is a $n \times n$ Hermitian or real symmetric matrix, the dominant $r$-dimensional eigenspace of $X$, denoted $\mathbf E_r(X)$, is the subspace of ~$\R^n$ or $\C^n$ spanned by eigenvectors of $X$ corresponding to its $r$ largest eigenvalues.
\end{definition}

Note that by this definition, if $\lambda_r(X) = \lambda_{r + 1}(X)$, then $\mathbf E_r(X)$ actually has dimension $> r$, but that will never be the case in this analysis.

\begin{definition}[Dominant rank-$r$ projector]
If $X$ is a $n \times n$ Hermitian or real symmetric matrix, the dominant rank-$r$ projector of $X$, denoted $P_r(X)$, is the orthogonal projection operator onto $\mathbf E_r(X)$.
\end{definition}
$P_r(X)$ is a rank-$r$, self-adjoint operator which acts as the identity on $\mathbf E_r(X)$.  It has $r$ eigenvalues equal to 1 and $n - r$ equal to 0.  If $\vv_1, \ldots, \vv_r$ is an orthonormal basis for $\mathbf E_r(X)$, then
\begin{equation}\label{eq:Pkcomp}
P_r(X) = \sum_{i = 1}^r\vv_i\vv_i^*,
\end{equation}
where $\vv^*$ denotes either the transpose or conjugate transpose of $\vv$, depending on whether we are working over $\R$ or $\C$. Let us define $Y$ to be the \emph{incidence matrix} of $\CC$; i.e.,
\begin{align}\label{eq:incidence}
Y_{uv} := \left\{
\begin{array}{ll}
1	& \textrm{if $u, v$ are in the same part of $\CC$},	\\
0	& \textrm{else}
\end{array}
\right. \end{align}
Thus, it is our goal to reconstruct $Y$ given $H \sim H(n, d, \CC, p, q)$.
\begin{theorem}\label{thm:proj}
Let $A, \expval A$, and $\tilde A$ be defined as in Sections~\ref{sec:weightedgraphs} and~\ref{sec:concentration}.  Then
\[P_k(\expval A) = P_k(\tilde A) = P_k(Y) = \frac1sY.\]
\end{theorem}
\begin{proof}
Let $\1_{C_i} \in \{0, 1\}^n$ denote the indicator vector for cluster $C_i$ and $J_n$ the $n \times n$ all ones matrix.  Then we can write
\[Y = \sum_{i = 1}^k\1_{C_i}\1_{C_i}^\top, \quad \tilde A = (a - b)Y + bJ_n, \quad \expval A = \tilde A - aI_n.\]
for some constants $a > b > 0$.  Thus, $\left\{\frac1{\sqrt s}\1_{C_i} : i = 1, \ldots, k\right\}$ is an orthonormal basis for the column space of both $Y$ and $\tilde A$, and hence, in accordance with~\eqref{eq:Pkcomp},
\[P_k(Y) = P_k(\tilde A) = \sum_{i = 1}^k\frac1s\1_{C_i}\1_{C_i}^\top = \frac1sY.\]
Now, observe that the eigenvalues of $\expval A$ are those of $\tilde A$ shifted down by $a$, and $\vv$ is an eigenvector of $\expval A$ if and only if it is an eigenvector of $\expval A$; hence, the dominant $k$-dimensional eigenspace of $\expval A$ is the same as the column space of $\tilde A$, and therefore $P_k(\expval A) = P_k(\tilde A)$.
\end{proof}

Thus, $P_k(\expval A) = P_k(\tilde A)$ gives us all the information we need to reconstruct $Y$.  Unfortunately, a SBM recovery algorithm doesn't have access to $\expval A$ or $\tilde A$ (if it did the problem would be trivial), but the following theorem shows that the random matrix $P_k(A)$ is a good approximation to $P_k(\expval A)$ and thus reveals the underlying rank-$k$ structure of $A$:

\begin{theorem}\label{thm:projdiff}
Assume~\eqref{eq:concentration} holds.  Then
\[\|P_k(A) - P_k(\expval A)\|_2 \leq \varepsilon\]
and
\[\|P_k(A) - P_k(\expval A)\|_\frob \leq \sqrt{2k}\varepsilon\]
for any $\displaystyle\epsilon \geq \frac{6d\sqrt{d{n \choose d - 1}}}{{s - 2 \choose d - 2}(p - q)s - 12d\sqrt{d{n \choose d - 1}}}$.
\end{theorem}

To prove Theorem \ref{thm:projdiff}, we use the following Lemma from ~\cite[Lemma~4]{cole2018recovering}.

\begin{lemma}\label{lem:col18}
	Let $X,Y\in\mathbb R^{n\times n}$ be symmetric. Suppose that the largest $k$ eigenvalues of both $X,Y$ are at least $\beta$, and the remaining $n-k$ eigenvalues of both $X,Y$ are at most $\alpha$, where $\alpha<\beta$. Then 
	\begin{align}
	\|P_k(X)-P_k(Y)\|_2\leq \frac{\|X-Y\|_2}{\beta-\alpha},\\
	\|P_k(X)-P_k(Y)\|_F\leq \frac{\sqrt{2k}\|X-Y\|_2}{\beta-\alpha}.	
	\end{align}

\end{lemma}

\begin{proof}[Proof of Theorem \ref{thm:projdiff}]
Apply Lemma \ref{lem:col18} with $X = A$, $Y = \expval A$ and
\[\alpha = {{s-2}\choose{d-2}}(p-q)(s-1)-{{n-2}\choose{d-2}}q - 6d\sqrt{d{{n}\choose{d-1}}},\]
\[\beta = -{{s-2}\choose{d-2}}(p-q)-{{n-2}\choose{d-2}}q + 6d\sqrt{d{{n}\choose{d-1}}}.\]
Note that in order for this to work we need $\alpha > \beta$, i.e.
\[{s - 2 \choose d - 2}(p - q)s > 12d\sqrt{d{n \choose d - 1}}.\]
\end{proof}

\section{Constructing an approximate cluster}\label{sec:approxcluster}

In this section we show how to use $P_k(A)$ to construct an ``approximate cluster'', i.e.\ a set with small symmetric difference with one of the clusters.  We will show that
\begin{itemize}
\item	If $|W| = s$ and $\|P_k(A)\1_W\|_2$ is large, then $W$ must have large intersection with some cluster (Lemma~\ref{lemma:largenorm=>approxcluster})
\item	Such a set $W$ exists among the sets $W_1, \ldots, W_n$, where $W_v$ is the indices of the $s - 1$ largest entries in column $v$ of $P_k(A)$, along with $v$ itself (Lemma~\ref{lemma:existsgood}).
\end{itemize}
The intuition is that if $\|P_k(A) - P_k(\expval A)\|_2 \leq \varepsilon$, then
\[\|P_k(A)\1_W\|_2^2 \approx \|P_k(\expval A)\1_W\|_2^2 = \frac1s\sum_{i = 1}^k|W \cap C_i|^2,\]
and this quantity is maximized when $W$ comes mostly from a single cluster $C_i$.

%
%
%
%
%
%

Lemmas~\ref{lemma:largenorm=>approxcluster} and~\ref{lemma:existsgood} below are essentially the same as Lemmas~18 and~17 in~\cite{cole2017simple}.  As $P_k(A) = \frac1s\sum_i\1_{C_i}\1_{C_i}^\top$ as in the graph case (Theorem~\ref{thm:proj}), we can import their proofs directly from the graph case.  However, we present a simpler proof for Lemma~\ref{lemma:largenorm=>approxcluster}.

\begin{lemma}\label{lemma:largenorm=>approxcluster}
Assume~\eqref{eq:concentration} holds and  $\displaystyle \frac{6d\sqrt{d{n \choose d - 1}}}{{s - 2 \choose d - 2}(p - q)s - 12d\sqrt{d{n \choose d - 1}}}\leq \epsilon\leq \frac{1}{12}$.  Let $|W| = s$ and $||P_k(A)\1_{W}||_2 \geq (1 - 2\varepsilon)\sqrt s$.  Then $|W \cap C_i| \geq  (1 - 6\varepsilon)s$ for some $i$. 
\end{lemma}

\begin{proof}
By Theorem~\ref{thm:projdiff},
$$\|(P_k(A)-P_k(\mathbb E A))\mathbf{1}_W\|_2\leq \epsilon\|\mathbf{1}_W\|_2=\epsilon \sqrt s.
$$
 And by the triangle inequality,
\begin{equation}\label{eq:triangle}
\|P_k(\expval A)\1_W\|_2 \geq \|P_k( A)\1_W\|_2-\epsilon \sqrt{s}\geq (1 - 2\varepsilon)\sqrt s - \varepsilon\sqrt{s} = (1 - 3\varepsilon)\sqrt s.
\end{equation}
We will show that in order for this to hold, $W$ must have large intersection with some cluster.

Fix $t$ such that $\frac{s}2 \leq t \leq s$.  Assume by way of contradiction that $|W \cap C_i| \leq t$ for all $i$.  Observe that by Theorem~\ref{thm:proj}
\begin{equation}\label{eq:sumofsquares}
\|P_k(\expval A)\1_W\|_2^2 = \frac1s\sum_{i = 1}^k|W \cap C_i|^2.
\end{equation}
Let $x_i = |W \cap C_i|$ and consider the optimization problem
\begin{eqnarray*}
\max			& & \frac1s\sum_{i = 1}^kx_i^2	\\
\textrm{s.t.}	& & \sum_{i = 1}^kx_i = s,	\\
				& & 0 \leq x_i \leq t \textrm{ for } i = 1, \ldots, k.
\end{eqnarray*}
It is easy to see that the maximum occurs when $x_i = t, x_j = s - t$ for some $i, j$, $x_l = 0$ for all $l \neq i, j$, and the maximum is $\frac{t^2}s + \frac{(s - t)^2}s$.  Thus, by~(\ref{eq:triangle}) and~(\ref{eq:sumofsquares}) we have
\[(1 - 3\varepsilon)^2s \leq \|P_k(\expval A)\1_W\|_2^2 \leq \frac{t^2}s + \frac{(s - t)^2}s.\]
Solving for $t$, this implies that
\[t \geq \left(\frac{1}2 + \frac12\sqrt{1 - 12\varepsilon + 18\varepsilon^2}\right)s > (1 - 6\varepsilon)s.\]

Thus, if we choose $t \in [s / 2, (1 - 6\varepsilon)s]$ we have a contradiction.  Let us choose $t$ to be as large as possible, $t = (1 - 6\varepsilon)s$.  Then it must be the case that $|W \cap C_i| \geq t = (1 - 6\varepsilon)s$ for some $i$.  Note that for the proof to go through we require $\frac{1}2 \leq 1 - 6\varepsilon$, which is satisfied if $\varepsilon \leq 1 / 12$.
\end{proof}

This lemma gives us a way to identify an ``approximate cluster'' using only $A$; however, it would take $\Omega(n^s)$ time to try all sets $W$ of size $s$.  However, if we define $W_v$ to be $v$ along with the indices of the $s - 1$ largest entries of column $v$ of $P_k(A)$ (as in Step~\ref{step:Wv} of Algorithm~\ref{alg}), then Lemma~\ref{lemma:existsgood} below will show that one of these sets satisfies the conditions of Lemma~\ref{lemma:largenorm=>approxcluster}; thus, we can produce an approximate cluster in polynomial time by taking the $W_v$ that maximizes $||P_k(A)\1_{W_v}||_2$.

\begin{lemma}\label{lemma:existsgood}
Assume~\eqref{eq:concentration} holds and $\displaystyle \epsilon \geq \frac{6d\sqrt{d{n \choose d - 1}}}{{s - 2 \choose d - 2}(p - q)s - 12d\sqrt{d{n \choose d - 1}}}$.  For $v = 1, \ldots, n$, let $W_v$ be defined as in Step~\ref{step:Wv} of Algorithm~\ref{alg}. Then there exists a column $v$ such that 
\begin{equation*}\label{goodcoldef}
\|P_k(A)\1_{W_v}\|_2 \geq (1 - 8\varepsilon^2 - \varepsilon)\sqrt s \geq (1 - 2\varepsilon)\sqrt s.
\end{equation*}

\end{lemma}

Lemmas~\ref{lemma:largenorm=>approxcluster} and~\ref{lemma:existsgood} together prove that, as long as~\eqref{eq:concentration} holds, Steps~\ref{step:proj}-\ref{step:approxcluster} successfully construct a set $W$ such that $|W| = s$ and $|W \cap C_i| \geq (1 - 6\varepsilon)s$ for some $i$.  In the following section we will see how to recover $C_i$ exactly from $W$.

\section{Exact recovery by counting hyperedges}\label{sec:exactrec}

Suppose we have a set $W \subset [n]$ such that $|W \triangle C_i| \leq \varepsilon s$ for some $i$ ($\triangle$ denotes symmetric difference).  In the graph case ($d = 2$) we can use $W$ to recover $C_i$ exactly w.h.p.\ as follows:
\begin{enumerate}
\item	Show that w.h.p.\ for any $u \in C_i$ will have at least $(p - \varepsilon)s$ neighbors in $C_i$, while any $v \notin C_i$ will have at most $(q + \varepsilon)s$ neighbors in $C_i$.  This follows from a simple Hoeffding argument.
\item	Show that that, if these bounds hold for any $u, v$, then (deterministically) any $u \in C_i$ will have at least $(p - 2\varepsilon)$ neighbors in $W$, while any $v \notin C_i$ will have at most $(q + 2\varepsilon)$ neighbors in $W$.  Thus, we can use number of vertices in $W$ to distinguish between vertices in $C_i$ and vertices in other clusters.
\end{enumerate}
See~\cite[Lemmas~19-20]{cole2017simple} for details.  The reason we cannot directly apply a Hoeffding argument to $W$ is that $W$ depends on the randomness of the instance $A$, thus the number of neighbors a vertex has in $W$ is not the sum of $|W|$ fixed random variables.

To generalize to hypergraphs with $d > 2$, an obvious analogue of the notion of number of neighbors a vertex $u$ has in a vertex set $W$ is to define the random variable
\[N_{u, W} := \sum_{e : u \in e, e \setminus \{u\} \subseteq W}T_e,\]
i.e.\ the number of hyperedges containing $u$ and $d - 1$ vertices from $W$.  When $d = 2$ this is simply the number of neighbors $u$ has in $W$ (see Figure \ref{fig:NuW} for the case $d=3$).  We get the following analogue to~\cite[Lemma~19]{cole2017simple}.

\begin{figure}
\centering
\includegraphics[width=5cm]{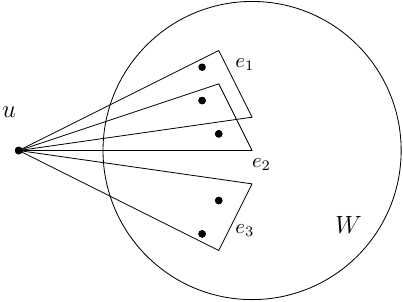}	
\caption{When $d=3$, $N_{u,W}$ is the number of hyperedges containing $u$ and $2$ vertices in $W$. In the figure above, $N_{u,W}=3$.}\label{fig:NuW}
\end{figure}

\begin{lemma}\label{lemma:hoeffding}
Consider cluster $C_i$ and vertex $u \in [n]$, and let $\varepsilon > 0$.  If $u \in C_i$, then for $n$ sufficiently large and $d=o(s)$,
\begin{equation}\label{jinCi}
N_{u, C_i} \geq (p-\varepsilon){s \choose d - 1}
\end{equation}
with probability $\geq 1 -  \exp\left(-\varepsilon^2{s-1 \choose d-1}\right)$, and if $u \notin C_i$, then
\begin{equation}\label{jnotinCi}
N_{u, C_i} \leq (q+\varepsilon){s \choose d - 1}
\end{equation}
with probability $\geq 1 -  \exp\left(-\varepsilon^2{s-1 \choose d-1}\right)$.
\end{lemma}

\begin{proof}
 For $u \in C_i$, $N_{u, C_i}$ is the sum of ${s - 1 \choose d - 1}$ independent Bernoulli random variables with expectation $p$, so Hoeffding's inequality yields
 \begin{align*}
 	\mathbb P\left(N_{u,C_i}\leq (p-\varepsilon){s  \choose d - 1} \right)
 	&= \mathbb P\left(N_{u,C_i}\leq \left(p-\frac{\varepsilon s-(d-1)p}{s-d+1}\right){s-1 \choose d-1}  \right)\\
 	&\leq \exp\left(-2\left(\frac{\varepsilon s-(d-1)p}{s-d+1}\right)^2{s-1 \choose d-1}\right )\\
 	&\leq \exp\left(-\varepsilon^2{s-1 \choose d-1}\right)
 \end{align*}
 Note that the last inequality holds for $d=o(s)$ and $n$ sufficiently large.
 
  For $v \notin C_i$, $N_{v, C_i}$ is the sum of ${s \choose d - 1}$ independent Bernoulli random variables with expectation $q$. So by Hoeffding's inequality again
  \begin{align*}
  	\mathbb P\left(N_{u,C_i}\geq  (q+\varepsilon){s  \choose d - 1} \right)&\leq \exp \left(-2\varepsilon^2 {s-1 \choose d-1}\right)\leq \exp\left(-\varepsilon^2 {s-1 \choose d-1}\right).
  \end{align*}
\end{proof}

The difficulty is in going from $N_{u, C_i}$ to $N_{u, W}$, where $W$ is a set such that $|W \triangle C_i| \leq \varepsilon s$.  We have the following estimate for $N_{u,W}$. 
\begin{lemma}\label{lemma:NuW}
	Let $W\subset [n]$ such that $|W|=s$ and $|W\cap C_i|\geq  (1-6\varepsilon)s$ for some $i\in[k]$. Then for $\epsilon<\frac{1}{16d}$, $d=o(s)$, and $n$ sufficiently large, we have the following:
	\begin{enumerate}
		\item If $j\in C_i$ satisfies \eqref{jinCi}, then
		$\displaystyle N_{j,W}\geq (p-16d\varepsilon){s \choose d-1},
		$
		\item If $j\not\in C_i$ satisfies \eqref{jnotinCi}, then
		$\displaystyle N_{j,W}\leq (q+16d \varepsilon){s \choose d-1}.
		$
	\end{enumerate}
\end{lemma}

	\begin{proof}
	Assume $j\in C_i$ and $j$ satisfies \eqref{jinCi}. As $|C_i|=s$, we have $|C_i\setminus W|\leq 6\varepsilon s$.
	
	 Let $\tilde{N}_{j,C_i\setminus W}$ be the number of hyperedges containing $j$ and $d-1$ vertices from $C_i$, among which at least one vertices from $C_i\setminus W$. We then have 
	 \begin{align*}
	 	N_{j,W}&\geq N_{j, W\cap C_i}=N_{j,C_i}- \tilde{N}_{j,C_i\setminus W}\\
	 	         &\geq  (p-\varepsilon){s\choose d-1}-\sum_{m=1}^{d-1}{\lceil 6\varepsilon s\rceil \choose m}{s \choose d-1-m}.
	 \end{align*}
	 In the inequality above, we bound $\tilde{N}_{j,C_i\setminus W}$ by a deterministic counting argument, i.e. we count all possible hyperedges that include a vertex $j$, with $m$ vertices from $C_i\setminus W$ and remaining $(d-1-m)$ vertices from $C_i$ for $1\leq m\leq d-1$.

	 Note that we can choose $\displaystyle \varepsilon<\frac{1}{16d}$ then for $n$ sufficiently large, we have 
	 \begin{align*}
	 	\sum_{m=1}^{d-2}{\lceil 6\varepsilon s\rceil \choose m}{s \choose d-1-m}
	 	&\leq \sum_{m=1}^{d-2}{  7\varepsilon s \choose m}{s \choose d-1-m}\\
	 	&\leq {s\choose d-1}\sum_{m=1}^{d-1}(7\varepsilon s)^m\frac{(s-d+1)!(d-1)!}{(s-d+1+m)!(d-1-m)!}\\
	 	&\leq {s \choose d-1}\sum_{m=1}^{d-1}\left(\frac{7\varepsilon s d}{s-d+2}\right)^m\\
	 	&\leq {s \choose d-1}\frac{14\varepsilon s d}{s-d+2}\leq 15d\varepsilon {s \choose d-1}
	 \end{align*}
	 
	 So we have 
	 \begin{align*}
	 	N_{j,W}\geq (p-16d\varepsilon){s\choose d-1}.
	 \end{align*}
	 
If $j\not\in C_i$, let $\tilde{N}_{j,W\setminus C_i}$ be the number of hyperedges containing $j$ and $d-1$ vertices from $W$, among which at least one vertices from $W\setminus C_i$. Recall $|W\setminus C_i|\leq 6\varepsilon s$.
\begin{align*}
N_{j,W}&\leq N_{j,C_i\cup W}=N_{j,C_i}+\tilde{N}_{j,W\setminus C_i}\\
&\leq (q+\varepsilon){s\choose{d-1}}+ 	\sum_{m=1}^{d-1}{\lceil 6\varepsilon s\rceil \choose m}{s \choose d-1-m}\\
&\leq (q+\varepsilon){s\choose{d-1}}+ 	\sum_{m=1}^{d-1}{ 7\varepsilon s \choose m}{s \choose d-1-m}\\
&\leq (q+16d\varepsilon){s\choose d-1}.
\end{align*}	
\end{proof}

Lemma \ref{lemma:NuW} gives us a way to distinguish vertices $j\in C_i$ and $j\not\in C_i$ provided
$p-16d\varepsilon>q+16d\varepsilon$.  

\section{Proof of algorithm's correctness}\label{sec:delandrec}

We now have all the necessary pieces to prove the correctness of Algorithm~\ref{alg} (Theorem~\ref{mainthm}).  The proof is roughly the same as that of~\cite[Theorem~4]{cole2017simple}.

\subsection{Proof of correctness of first iteration}

Lemmas~\ref{lemma:largenorm=>approxcluster}-\ref{lemma:NuW} above prove that Steps~\ref{step:adjmatrix}-\ref{step:exactrec} of Algorithm~\ref{alg} correctly recover a single cluster in the first iteration.

\begin{theorem}\label{thm:1stitr}
Assume that~\eqref{eq:concentration} holds and that for $i = 1, \ldots, k$, \eqref{jinCi} holds for all $u \in C_i$ and~\eqref{jnotinCi} holds for all $u \notin C_i$ with  $$\displaystyle \frac{6d\sqrt{d{n \choose d - 1}}}{{s - 2 \choose d - 2}(p - q)s - 12d\sqrt{d{n \choose d - 1}}}\leq \epsilon\leq \frac{p - q}{32d}.$$  Then Steps~\ref{step:adjmatrix}-\ref{step:exactrec} of Algorithm~\ref{alg} exactly recover a cluster $C_i$ in the first iteration. 

\end{theorem}

\begin{proof}
By Lemma~\ref{lemma:existsgood}, the set $W$ constructed in Step~\ref{step:approxcluster} has $||P_k(A)\1_W||_2 \geq (1 - 2\varepsilon)\sqrt s$.  By Lemma~\ref{lemma:largenorm=>approxcluster} (noting that $\epsilon \leq \frac{p - q}{32d} \leq \frac1{12}$), $|W \cap C_i| \geq (1 - 6\varepsilon)s$ for some $i$.  And by Lemma~\ref{lemma:NuW}, $N_{u, W} \geq (1 - 16\varepsilon)s$ for all $u \in C_i$, while $N_{u, W} \leq (q + 16d\varepsilon)s$ for all $u \notin C_i$.  If $\displaystyle \varepsilon < \frac{p - q}{32d}$, then $(p - 16d\varepsilon)s > (q + 16d\varepsilon)s$.  Thus, when we take the $s$ vertices $u$ with highest $N_{u, W}$ in Step~\ref{step:exactrec}, for each of them we have
\[N_{u, W} \geq (p - 16d\varepsilon)s > (q + 16d\varepsilon)s,\]
so none of them could possibly come from $[n] \setminus C_i$.  Therefore, the set $C$ constructed in Step~\ref{step:exactrec} must be equal to $C_i$.
\end{proof}

\subsection{The ``delete and repeat'' step}\label{sec:delete}

The difficulty with proving the success of Algorithm~\ref{alg} beyond the first iteration is that the iterations cannot be handled independently: whether or not the $t$-th iteration succeeds determines which vertices will be left in the $(t + 1)$-st iteration, which certainly affects whether or not the $(t + 1)$-st iteration succeeds.  However, notice that there is nothing probabilistic in the statement or proof of Theorem~\ref{thm:1stitr}: if certain conditions are true, then the first iteration of Algorithm~\ref{alg} will definitely recover a cluster.  In fact, the only probabilistic statements thus far are in Theorem~\ref{concentration} and Lemma~\ref{lemma:hoeffding}.  Similar to the analysis in~\cite{cole2017simple, cole2018recovering}, we will show that if certain (exponentially many) conditions are met, then \emph{all} iterations of Algorithm~\ref{alg} will succeed.  We will then show that all of these events occur simultaneously w.h.p.; hence, Algorithm~\ref{alg} recovers all clusters w.h.p.

We begin by introducing some terminology:
\begin{definition}[Cluster subhypergraph, cluster subtensor]
We define a cluster subhypergraph to be a subhypergraph of $H$ induced by a subset of the clusters $C_1, \ldots, C_k$.  Similarly, we define a cluster subtensor to be the principal subtensor of $T$ formed by restricting the indices to a subset of the clusters.  For $J \subseteq [k]$, we denote by $H^{(J)}$ the subhypergraph of $H$ induced by $\bigcup_{j \in J}C_j$, and we denote by $T^{(J)}$ the principal subtensor of $T$ with indices restricted to $\bigcup_{j \in J}C_j$.
\end{definition}

We now define two types of events on our probability space $H(n, d, \CC, p, q)$:
\begin{itemize}
\item	\emph{Spectral events} -- For $J \subseteq [k]$, let $E_J$ be the event that 
\[\|B - \expval B\|_2 \leq 6d\sqrt{d{m \choose d - 1}},\]
where $B$ is the adjacency matrix of $H^{(J)}$ and $m = s|J|$ the number of vertices in $H^{(J)}$.  Note that $B$ is \emph{not} simply a submatrix of $A$, as only a subset of the edges of $H$ are counted when computing the entries of $B$.
\item	\emph{Degree events} -- For $1 \leq i \leq k, 1 \leq u \leq n$, let $D_{i, u}$ be the event that $\displaystyle N_{u, C_i} \geq (p - \varepsilon){s \choose d - 1}$ if $u \in C_i$, or the event that $\displaystyle N_{u, C_i} \leq (q + \varepsilon){s \choose d - 1}$ if $u \notin C_i$.  These are the events that each vertex $u$ has approximately the correct value of $N_{u, C_i}$ for each cluster $C_i$.
\end{itemize}

Observe that there are $2^k$ spectral events and $nk$ degree events. We will now show that if all of these events occur, then Algorithm~\ref{alg} will definitely succeed in recovering all clusters.  Again, there is nothing probabilistic in this theorem or its proof.

\begin{lemma}\label{lem:allevents=>success}
Assume that $E_J$ holds for all $J \subseteq [k]$ and $D_{i, u}$ holds for all $i \in [k], u \in [n]$ with \begin{equation}\label{eq:epsilonub}
 \frac{6d\sqrt{d{n \choose d - 1}}}{{s - 2 \choose d - 2}(p - q)s - 12d\sqrt{d{n \choose d - 1}}}\leq \epsilon\leq \frac{p - q}{32d}.
\end{equation}  Then Algorithm~\ref{alg} recovers $C_1, \ldots, C_k$ exactly.

\end{lemma}

\begin{proof}[Proof sketch]
We omit the full proof as it is analogous to the proof in~\cite[Section~7.3]{cole2017simple}.  Essentially, we prove by induction that the $t$-th iteration succeeds for $t = 1, \ldots, k$.  If the 1st through $t$th iterations succeed, then the $(t + 1)$-st iteration receives as input a cluster sub-hypergraph $H^{(J)}$, for some $J \subseteq [k]$.  Hence, $E_J$ and $D_{i, u}$ for $i \in J$, $u \in \bigcup_{j \in J}C_j$ ensure the success of the $(t + 1)$-st iteration.  Note that if there are $m = |J|s$ vertices remaining, then Theorem~\ref{thm:projdiff} requires that
\[\varepsilon \geq \frac{6d\sqrt{d{m \choose d - 1}}}{{s - 2 \choose d - 2}(p - q)s - 12d\sqrt{d{m \choose d - 1}}},\]
but the this bound is largest when $m = n$; thus, the condition~\eqref{eq:epsilonub} is sufficient for all iterations.
\end{proof}

Finally, we show that all of the $E_J$ and $D_{i, u}$ hold simultaneously w.h.p.

\begin{lemma}\label{lem:alleventswhp}
$E_J$ and $D_{i, u}$ hold simultaneously for all $J \subseteq [k], i \in [k], u \in [n]$ with probability $\geq 1 - 2^k \cdot \exp(-s) - nk \cdot \exp\left(-\varepsilon^2{s - 1 \choose d - 1}\right)$ for any $\epsilon$ satisfying condition ~\eqref{eq:epsilonub}.
\end{lemma}

\begin{proof}
For any \emph{fixed} $J \subseteq [k]$, $H^{(J)}$ is simply an instance of a smaller HSBM; it has distribution $H(|J|s, d, \bigcup_{j \in J}C_j, p, q)$. Thus,
\[\prb\left(\overline{ E_J}\right) \leq \exp(-|J|s) \leq \exp(-s)\]
by Theorem~\ref{concentration}.  And for any $i \in [k], u \in [n]$,
\[\prb \left(\overline{ D_{i, u}}\right) \leq \exp\left(-\varepsilon^2{s - 1 \choose d - 1}\right)\]
by Lemma~\ref{lemma:hoeffding}.  The proof is completed by taking a union bound over all $J \subseteq [k], i \in [k], u \in [n]$.
\end{proof}

Theorem~\ref{mainthm} follows as an immediate corollary to Lemmas~\ref{lem:allevents=>success} and~\ref{lem:alleventswhp}.

\section{Lower bounds for SBM recovery}\label{sec:lowerbounds}

There have been many results which show that, for a fixed number of blocks, (H)SBM recovery becomes impossible if the edge probabilities are below a certain threshold.  For exact recovery of HSBMs with two blocks, it was shown in  \cite{lin2017fundamental,chien2018minimax,chien2018community} that the phase transition from impossible to possible occurs in the regime of logarithmic average degree by analyzing the minimax risk, and the exact threshold was given in \cite{kim2018stochastic}, by a generalization of the techniques in \cite{abbe2016exact} for graph SBMs. For the detection problem of HSBMs in the bounded expected degree regime, a phase transition was conjectured in \cite{angelini2015spectral} based on belief propagation and the non-backtracking operator. Very recently, a spectral method based on the self-avoiding walk was proved to achieve the conjectured threshold \cite{pal2019community}.  

For graph SBMs, a phase transition for the detection problem  in the bounded expected degree regime with finite many blocks, called Kesten-Stigum threshold, was conjectured in \cite{decelle2011asymptotic} and then proved in \cite{mossel2018proof,massoulie2014community,mossel2015reconstruction} for the 2-block case. Below the Kesten-Stigum threshold, no algorithms (even with exponential running time) will solve the detection problem, while above the threshold detection is not only possible but can be done in polynomial time (see \cite{abbe2018community} for further details).  The phase transition behavior  for the exact recovery problem with two blocks was proved in \cite{abbe2016exact}. Above the threshold, there are polynomial algorithms that solve the problem \cite{abbe2016exact,mossel2016consistency,abbe2017entrywise}.  For general SBMs with finite or a growing number of blocks, the minimax lower bounds were given in \cite{zhang2016minimax,chen2016statistical,jalali2016exploiting}.

Relatively little is known in the dense regime, even in the graph case, and most results focus on the related \emph{planted clique} problem~\cite{alon1998finding,jerrum1992large} rather than SBM recovery (a.k.a.\ planted partition).  It is generally believed that these problems become intractable when the size of the clique/clusters is $o(\sqrt n)$.  In this case, one would ideally like to prove that these problems are distNP-complete (where distP and distNP are distriubtional analogues of P and NP~\cite{levin1986average}).  However, showing the existence of a ``natural'' distNP-complete problem is itself a long-outstanding problem in complexity theory~\cite{arora2009computational}.

Instead, various authors have shown that certain types of algorithms will provably fail if the size of the planted clique/clusters is too small.  The first such result dates back to the original paper introducing the planted clique problem~\cite{jerrum1992large}, in which Jerrum showed that the Metropolis-Hastings algorithm fails to recover planted cliques of size $n^{1/2 - \epsilon}$ for any $\epsilon > 0$.  It was subsequently shown in~\cite{chen2016statistical, feige2000finding} that certain optimization-based approaches also fail in this regime, while Feldman et al.\ showed that \emph{statistical algorithms} also run into the same barrier~\cite{feldman2017statistical}. (A statistical algorithm is an algorithm which, instead of receiving samples from a distribution, can query an oracle for statistics on the distribution within some tolerance.  Such algorithms can be used to simulate many standard algorithms on randomized input.)  

Extending these results for planted clique to dense SBM and HSBM recovery appears to be a promising direction for future work.  In the $d$-uniform hypergraph case, it is unclear whether the barrier should be $\sqrt n$, $n^{1/d}$ or something else.  It seems plausible that the barrier could be less than $\sqrt n$ for $d > 2$, since $d$-uniform HSBMs have ${n \choose d}$ independent random variables in play compared with only ${n \choose 2}$ in the graph case, and thus we may expect certain random variables (e.g.\ degrees) to be more tightly concentrated about their expectations.  However, it seems doubtful that a spectral algorithm could do better than $\sqrt n$ using only the adjacency matrix (see Definition~\ref{adj}), as this is the error term introduced by concentration results for the spectral norm of a random symmetric $n \times n$ matrix (see, e.g.,~\cite{vershynin2018high}); to break the $\sqrt n$ barrier, we suspect that one must use the spectral properties of the \emph{adjacency tensor} and not simply reduce to the adjacency matrix.

While proving lower bounds for \emph{efficient} HSBM recovery appears to be difficult, one can readily prove that exact recovery is impossible for \emph{any} algorithm, regardless of running time, if the cluster size is small enough (i.e., an information theoretic lower bound). We will follow the proof ideas from \cite{chen2016statistical,jalali2015relative} in the graph SBM cases to prove an information theoretic lower bound for HSBMs.

Recall the definition of incidence matrix of a partition \eqref{eq:incidence}. Conversely, let $\mathcal C(Y)$ denote the partition of $[n]$ whose incidence matrix is $Y$.
Let $\mathcal Y$ be the set of all incidence matrices corresponding to  partitions of $[n]$ into $k$ parts of size $s$:
\begin{align}
\mathcal Y :=\{ Y: \exists \text{  $k$ clusters of size $s$ such that $Y$ is the corresponding incidence matrix} \}.	
\end{align}
In addition, recall that the Kullback-Leibler divergence between two Bernoulli random variables with means $u$ and $v$ is given by 
\begin{align}
D(u\| v)=u\log\frac{u}{v}+(1-u)\log \frac{1-u}{1-v}.	
\end{align}
We are now able to state our theorem for the lower bound  on the minimax error probability of recovering $Y^*$.

\begin{theorem}\label{thm:lowerbound}
If $128 \leq s \leq n / 2$ and 
\begin{equation}\label{eq:KL}
{s - 1 \choose d - 1}\max\{D(p \| q), D(q \| p)\} \leq \frac1{24}\ln(n - s),
\end{equation}
then
\[\inf_{f}\sup_{\mathcal Y^* \in \mathcal Y}\prb(f(T) \neq Y^*) \geq \frac12.\]
The infimum is taken over all measurable functions $f : \mathbf S^d(\{0, 1\}^n) \to \mathcal Y$, where $\mathbf S^d(\{0, 1\}^n)$ denotes the set of symmetric $d$-tensors in $\{0, 1\}^{n^d}$, and the probability is taken over a random adjacency tensor $T$ sampled from the HSBM distribution corresponding to $Y^*$, i.e. $T \sim H(n, \mathcal C(Y^*), p, q, d)$.
\end{theorem}

From this theorem we know that when $p,q$ are constant, if $s = O\left(\log^{\frac1{d - 1}}n\right)$, then for \emph{any} algorithm there is some ``bad'' input on which the algorithm will fail with probability at least $1 / 2$. When $d=2$, this is the result obtained in \cite{chen2016statistical}. It is unclear at present whether the exponent $\frac1{d - 1}$ can be improved. 

Note that Theorem \ref{thm:lowerbound} is an information theoretical lower bound.  On the other hand, our Theorem \ref{mainthm} considers only polynomial time solvability, and there is a considerable gap between the performance guarantee of Theorem~\ref{mainthm} and the lower bound given in Theorem~\ref{thm:lowerbound}.  Closing this gap remains an open problem.

\begin{proof}[Proof of Theorem \ref{thm:lowerbound}]
	Let $m=n-s$ and $\overline{\mathcal Y}=\{ Y_0,\dots, Y_m\}$ be a subset of  $\mathcal Y$ of size $m+1$ defined as follows.  Let $Y_0$ be the incidence matrix of the partition $C_1, \ldots, C_k$, where $C_l := \{(l - 1)s + 1, \ldots, ls\}$.  We then define $Y_i$ for $i > 0$ by swapping the cluster membership of $s$ and $s + i$.  More formally, if $s + i \in C_l$, then $Y_i$ is the incidence matrix of the partition $C_1', \ldots, C_k'$, where $C_1' := C_1 \cup \{s + i\} \setminus \{s\}$, $C_l' := C_l \cup \{s\} \setminus \{s + i\}$, and $C_j' := C_j$ for all $j \neq 1, l$.
	
	Let $\mathcal P_{(Y^*,T)}$ be the joint distribution of $(Y^*,T)$ where we first sample an incidence matrix $Y^*$ from $\overline{\mathcal Y}$ uniformly at random and then sample a hypergraph adjacency tensor $T \sim H(n, \mathcal C(Y^*), p, q, d)$ (see Section~\ref{sec:hsbmintro}). Then we have 
	\begin{align}\label{eq:lowerboundprob}
	\inf_{f}\sup_{Y^*\in\mathcal Y}\mathbb P(f(T)\not=Y^*)\geq \inf_{f} \mathbb P_{(Y^*, T)}	(f(T)\not=Y^*)\geq 1-\frac{I(Y^*;T)+1}{\log |\overline{\mathcal Y}|},
	\end{align}
where the last inequality is by Fano's inequality and $I(Y^*; T)$ is the mutual information between $Y^*$ and $T$. Let $\mathbb P_i$ be the probability distribution of the hypergraph $H$ conditioned on $Y^*=Y_i$. By the convexity of KL-divergence we have 
\begin{align*}
I(Y^*;T)\leq \frac{1}{(m+1)^2}\sum_{i,i'=0}^m D(\mathbb P_i\| \mathbb P_{i'})\leq \max_{i,i'} D(\mathbb P_i \| \mathbb P_{i'}).
\end{align*}
Note that $\mathbb P_i$ is the product of $n\choose d$ many Bernoulli distributions.  Let $\mathbb P_i(e)$ be the probability distribution of the hyperedge $e$ under the distribution $\mathbb P_i$, which is either $\text{Ber}(p)$ or $\text{Ber}(q)$. Then for any $i\not=i'$  we have 
\begin{align}
	D(\mathbb P_i \| \mathbb P_{i'})&\leq \sum_{e}D(\mathbb P_i(e)\| \mathbb P_{i'}(e))\leq 3 {s - 1 \choose d-1} D(p\| q)+3 {s - 1 \choose d-1} D(q\| p) \label{eq:KL_divergence_bound}\\
	&\leq 6 {s - 1\choose{d-1}}\max\{ D(p\|q), D(q\|p)\}.\notag 
\end{align}

Here the first inequality in \eqref{eq:KL_divergence_bound} is due to the fact that KL divergence is additive for products of independent distributions.  The second inequality comes from counting terms in which $\mathbb P_i(e) \neq \mathbb P_{i'}(e)$.  In the worst case we have $s + i \in C_l$ and $s + i' \in C_{l'}$ for some $l' \neq l \neq 0$, in which we get a contribution of $D(p || q)$ from all $e$ containing $s + i$ and $d - 1$ indices from $C_1 \setminus \{s\}$, all $e$ containing $s$ and $d - 1$ indices from $C_l \setminus \{s + i\}$, and all $e$ containing $s + i'$ and $d - 1$ indices from $C_{l'} \setminus \{s + i'\}$, a total of $3{s - 1 \choose d - 1}$ terms; we get a contribution of $D(q || p)$ from the same number of terms.

If \eqref{eq:KL} holds, then $I(Y^*;T)\leq \frac{1}{4} \log(n-s)=\frac{1}{4} \log |\overline{\mathcal Y}|.$ When $n\geq 128$, we have $\log |\overline{\mathcal Y}|\geq 4$. Then from \eqref{eq:lowerboundprob} the minimax error probability is at least $1/2$. This completes the proof.
\end{proof}

\appendix

\section{Simple counting algorithm}\label{sec:counting}

One can recover HSBMs by simply counting the number of hyperedges containing pairs of vertices: with high probability, pairs of vertices in the same cluster will be contained in more hyperedges than pairs in different clusters.  However, our spectral algorithm provides better performance guarantees than this simple counting algorithm.

\begin{algorithm}
\caption{}\label{alg:counting}
Given $H = (V, E)$, $|V| = n$, number of clusters $k$, and cluster size $s = n / k$:
\begin{enumerate}
\item	For each pair of vertices $u \neq v$, compute $A_{uv} :=$ number of hyperedges containing $u$ and $v$.
\item	For each vertex $v$, let $W_v$ be the set of vertices containing $v$ and the $s - 1$ vertices $u \neq v$ with highest $A_{uv}$ (breaking ties arbitrarily).  It will be shown that w.h.p.\ $W_v$ will be the cluster $C_i$ containing $v$.
\end{enumerate}
\end{algorithm}

\begin{theorem}\label{thm:counting}
Let $H$ be sampled from $H(n, d, \CC, p, q)$, where $d\geq 3,\CC = \{C_1, \ldots, C_k\}$ and $|C_i| = s = n / k$ for $i = 1, \ldots, k$.  Then Algorithm \ref{alg:counting} recovers $\CC$ with probability $\geq 1 - 1 / n$ if 
\[{s - 2 \choose d - 2}(p - q) > \sqrt{6{n - 2 \choose d - 2}\log n}.\]
\end{theorem}

A simple counting algorithm for graph SBMs was given in \cite{chen2016statistical}. Our algorithm is modified from \cite{chen2016statistical} for hypergraphs  based on counting hyperedges and it requires $d\geq 3$.
\begin{proof}
For each $u \neq v$, $A_{uv}=\displaystyle  \sum_{e: u, v\in e}T_e$ is the sum of ${n - 2 \choose d - 2}$ independent Bernoulli random variables of expectation either $p$ or $q$.  Thus, it follows from a straightforward application of Hoeffding's inequality that
\begin{equation}\label{eq:countinglb}
A_{uv} \geq {n - 2 \choose d - 2}q + {s - 2 \choose d - 2}(p - q) - \sqrt{\frac32{n - 2 \choose d - 2}\log n}
\end{equation}
with probability $\geq 1 - 1 / n^3$ if $u$ and $v$ are in the same cluster and 
\begin{equation}\label{eq:countingub}
A_{uv} \leq {n - 2 \choose d - 2}q + \sqrt{\frac32{n - 2 \choose d - 2}\log n}
\end{equation}
with probability $\geq 1 - 1 / n^3$ if $u$ and $v$ are in different clusters.  Taking a union bound over all ${n \choose 2}$ pairs, these bounds hold for all pairs $u \neq v$ with probability $\geq 1 - 1 / n$.  Thus, as long as the lower bound in~\eqref{eq:countinglb} is greater than the upper bound in~\eqref{eq:countingub}, for each $v$ the $s - 1$ vertices with highest $A_{uv}$ will be the other vertices in $v$'s cluster.

\end{proof}

In particular, if we bound the binomial coefficient ${a \choose b}$ by $\left(\frac ab\right)^b \leq {a \choose b} \leq \left(\frac{ae}b\right)^b$, we see that
\[s \geq c_1\sqrt{nd}\left(\frac{\sqrt{\log n}}{p - q}\right)^{\frac1{d - 2}}\]
and
\[p - q \geq \frac{c_2(2end)^{\frac{d - 2}2}\sqrt{\log n}}{s^{d - 2}} = c_2\left(\frac{2ek^2d}{n}\right)^{\frac{d - 2}2}\sqrt{\log n}\]
are both sufficient conditions for recovery, where $c_1$ and $c_2$ are absolute constants.

\section{The sparse case}\label{sec:sparse}

We can also analyze the performance of Algorithm~\ref{alg} in the sparse case, in which we treat $k,d$ as fixed and try to make $p$ and $q$ as small as possible.  Our concentration bound \eqref{concentration} is not optimal in the sparse case. However, when $p=\displaystyle \frac{\omega(\log^4 n)}{n^{d-1}}$, we can still get a good concentration inequality of the adjacency matrix $A$ using Lemma 5 in~\cite{lu2012loose}. We include it here:
\begin{lemma}
	If $\displaystyle p=\frac{\omega(\log^4n)}{n^{d-1}},$ we have 
	\begin{align}\label{sparsecon}
	\|A-\mathbb EA\|_2\leq 2d\sqrt{n^{d-1}p} 	
	\end{align}
with probability $1-o(1)$.
\end{lemma}

In this case, we get the following analog of Theorem \eqref{thm:projdiff}.

\begin{lemma}
	
Assume~\eqref{sparsecon} holds.  Then
\[\|P_k(A) - P_k(\expval A)\|_2 \leq \varepsilon\]
and
\[\|P_k(A) - P_k(\expval A)\|_\frob \leq \sqrt{2k}\varepsilon\]
for any \begin{align}\label{epsilonLow}
 \epsilon \geq \frac{2d\sqrt{n^{d-1}p}}{{s - 2 \choose d - 2}(p - q)s - 4d\sqrt{n^{d-1}p} }.\end{align}
\end{lemma}
\begin{proof}
Apply Lemma \ref{lem:col18} with $X = A$, $Y = \expval A,$ and 
\[\alpha = {{s-2}\choose{d-2}}(p-q)(s-1)-{{n-2}\choose{d-2}}q - 2d\sqrt{n^{d-1}p},\]
\[\beta = -{{s-2}\choose{d-2}}(p-q)-{{n-2}\choose{d-2}}q + 2d\sqrt{n^{d-1}p}.\]
Note that in order for this to work we need 
\begin{align}\label{sparsePQ}
{s - 2 \choose d - 2}(p - q)s > 4d\sqrt{n^{d-1}p}.	
\end{align}
\end{proof}

If we assume $\displaystyle p=\frac{\omega(\log^4n)}{n^{d-1}}, p-q=\Theta(p)$, and $k$ is fixed, condition \eqref{sparsePQ} always holds.
In addition, we want the failure probability to be $o(1)$, so we require 
\[\exp\left(-\epsilon^2{s - 1 \choose d - 1}\right) = o((nk)^{-1}).\]
Putting
$\displaystyle \epsilon^2{s - 1 \choose d - 1} \geq 3\log  n$
suffices to accomplish this. Therefore, we require that
$\displaystyle
\epsilon\geq \frac{c_4\sqrt{\log n}}{n^{(d-1)/2}}	
$ for some constant $c_4$ depending only on $d,k$ as an additional lower bound on $\epsilon$.
 On the other hand, to make the algorithm succeed, we need to have $\displaystyle \epsilon<\frac{p-q}{32d}$ from the analysis in Section \ref{sec:delandrec}. Together we have the following constraint on $\epsilon$:
 \begin{align}\label{sparseMax}
 &\max\left\{\frac{c_4\sqrt{\log n}}{n^{(d-1)/2}},\frac{2d\sqrt{n^{d-1}p}}{{s - 2 \choose d - 2}(p - q)s - 4d\sqrt{n^{d-1}p} }\right\}<\epsilon<\frac{p-q}{32d}.
 \end{align}
To make \eqref{sparseMax} work, assuming $\displaystyle p-q>c_5p$ for some constant $0<c_5<1$, we have 
$$\displaystyle 
 	p-q\geq \frac{c_6}{n^{(d-1)/3}}
$$
for some constant $c_6>0$ depending on $d,k$ and $c_5$. This yields the following corollary to Theorem~\ref{mainthm}:


%
%

\begin{theorem}[Sparse case] Let $k,d$ be constant and 
let $H$ be sampled from $H(n, d, \CC, p, q)$, where $\CC = \{C_1, \ldots, C_k\}$ and $|C_i| = s = n / k$ for $i = 1, \ldots, k$.  If $p-q>c_5p$ for some constant $0<c_5<1$ and 
\begin{align}
 	p-q\geq \frac{c_6}{n^{(d-1)/3}}
 \end{align}
 for some constant $c_6$ depending on $d,k$ and $c_5$, 
then Algorithm~\ref{alg} recovers $\CC$ w.h.p.
\end{theorem} 
%
%

Thus, we see that our algorithm is far from optimal in the sparse case: the algorithms developed in~\cite{kim2018stochastic,lin2017fundamental,chien2018minimax,chien2018community,ghoshdastidar2017consistency} all provide better performance guarantees.  In fact, even the trivial hyperedge counting algorithm (Algorithm~\ref{alg:counting}) beats our spectral algorithm in the sparse case.

\section*{Acknowledgements} The authors would  thank MSRI 2018 Summer School: Representations of High Dimensional Data, during which this project was initiated. The authors are grateful to anonymous referees for their detailed comments and suggestions, which have improved the quality of this
 paper.  S.C.\ is supported by a PIMS Postdoctoral Fellowship. Y.Z.\ is  supported by NSF DMS-1712630.

 \bibliographystyle{plain}
\bibliography{globalref.bib}

\begin{thebibliography}{10}

\bibitem{abbe2018community}
Emmanuel Abbe.
\newblock Community detection and stochastic block models: Recent developments.
\newblock {\em Journal of Machine Learning Research}, 18(177):1--86, 2018.

\bibitem{abbe2016exact}
Emmanuel Abbe, Afonso~S Bandeira, and Georgina Hall.
\newblock Exact recovery in the stochastic block model.
\newblock {\em IEEE Transactions on Information Theory}, 62(1):471--487, 2016.

\bibitem{abbe2017entrywise}
Emmanuel Abbe, Jianqing Fan, Kaizheng Wang, and Yiqiao Zhong.
\newblock Entrywise eigenvector analysis of random matrices with low expected
  rank.
\newblock {\em arXiv preprint arXiv:1709.09565}, 2017.

\bibitem{ahn2016community}
Kwangjun Ahn, Kangwook Lee, and Changho Suh.
\newblock Community recovery in hypergraphs.
\newblock In {\em Communication, Control, and Computing (Allerton), 2016 54th
  Annual Allerton Conference on}, pages 657--663. IEEE, 2016.

\bibitem{ahn2018hypergraph}
Kwangjun Ahn, Kangwook Lee, and Changho Suh.
\newblock Hypergraph spectral clustering in the weighted stochastic block
  model.
\newblock {\em IEEE Journal of Selected Topics in Signal Processing}, 2018.

\bibitem{alistarh2015streaming}
Dan Alistarh, Jennifer Iglesias, and Milan Vojnovic.
\newblock Streaming min-max hypergraph partitioning.
\newblock In {\em Advances in Neural Information Processing Systems}, pages
  1900--1908, 2015.

\bibitem{alon1998finding}
Noga Alon, Michael Krivelevich, and Benny Sudakov.
\newblock Finding a large hidden clique in a random graph.
\newblock {\em Random Structures \& Algorithms}, 13(3-4):457--466, 1998.

\bibitem{ames2014guaranteed}
Brendan~PW Ames.
\newblock Guaranteed clustering and biclustering via semidefinite programming.
\newblock {\em Mathematical Programming}, 147(1-2):429--465, 2014.

\bibitem{angelini2015spectral}
Maria~Chiara Angelini, Francesco Caltagirone, Florent Krzakala, and Lenka
  Zdeborov{\'a}.
\newblock Spectral detection on sparse hypergraphs.
\newblock In {\em Communication, Control, and Computing (Allerton), 2015 53rd
  Annual Allerton Conference on}, pages 66--73. IEEE, 2015.

\bibitem{arora2009computational}
Sanjeev Arora and Boaz Barak.
\newblock {\em Computational Complexity: a Modern Approach}.
\newblock Cambridge University Press, 2009.

\bibitem{bordenave2018nonbacktracking}
Charles Bordenave, Marc Lelarge, and Laurent Massouli{\'e}.
\newblock Nonbacktracking spectrum of random graphs: Community detection and
  nonregular ramanujan graphs.
\newblock {\em The Annals of Probability}, 46(1):1--71, 2018.

\bibitem{bretto2005hypergraph}
Alain Bretto and Luc Gillibert.
\newblock Hypergraph-based image representation.
\newblock In {\em International Workshop on Graph-Based Representations in
  Pattern Recognition}, pages 1--11. Springer, 2005.

\bibitem{bulo2009game}
Samuel~R Bul{\`o} and Marcello Pelillo.
\newblock A game-theoretic approach to hypergraph clustering.
\newblock In {\em Advances in neural information processing systems}, pages
  1571--1579, 2009.

\bibitem{chen2014improved}
Yudong Chen, Sujay Sanghavi, and Huan Xu.
\newblock Improved graph clustering.
\newblock {\em IEEE Transactions on Information Theory}, 60(10):6440--6455,
  2014.

\bibitem{chen2016statistical}
Yudong Chen and Jiaming Xu.
\newblock Statistical-computational tradeoffs in planted problems and submatrix
  localization with a growing number of clusters and submatrices.
\newblock {\em The Journal of Machine Learning Research}, 17(1):882--938, 2016.

\bibitem{chien2018community}
I~Chien, Chung-Yi Lin, and I-Hsiang Wang.
\newblock Community detection in hypergraphs: Optimal statistical limit and
  efficient algorithms.
\newblock In {\em International Conference on Artificial Intelligence and
  Statistics}, pages 871--879, 2018.

\bibitem{chien2018minimax}
I~Chien, Chung-Yi Lin, and I-Hsiang Wang.
\newblock On the minimax misclassification ratio of hypergraph community
  detection.
\newblock {\em arXiv preprint arXiv:1802.00926}, 2018.

\bibitem{cole2018recovering}
Sam Cole.
\newblock Recovering nonuniform planted partitions via iterated projection.
\newblock {\em Linear Algebra and its Applications}, 2018.

\bibitem{cole2017simple}
Sam Cole, Shmuel Friedland, and Lev Reyzin.
\newblock A simple spectral algorithm for recovering planted partitions.
\newblock {\em Special Matrices}, 5(1):139--157, 2017.

\bibitem{decelle2011asymptotic}
Aurelien Decelle, Florent Krzakala, Cristopher Moore, and Lenka Zdeborov{\'a}.
\newblock Asymptotic analysis of the stochastic block model for modular
  networks and its algorithmic applications.
\newblock {\em Physical Review E}, 84(6):066106, 2011.

\bibitem{ducournau2012reductive}
Aur{\'e}lien Ducournau, Alain Bretto, Soufiane Rital, and Bernard Laget.
\newblock A reductive approach to hypergraph clustering: An application to
  image segmentation.
\newblock {\em Pattern Recognition}, 45(7):2788--2803, 2012.

\bibitem{feige2000finding}
Uriel Feige and Robert Krauthgamer.
\newblock Finding and certifying a large hidden clique in a semirandom graph.
\newblock {\em Random Struct. Algorithms}, 16:195--208, 2000.

\bibitem{feldman2017statistical}
Vitaly Feldman, Elena Grigorescu, Lev Reyzin, Santosh~S. Vempala, and Ying
  Xiao.
\newblock Statistical algorithms and a lower bound for detecting planted
  cliques.
\newblock {\em J. ACM}, 64(2):8:1--8:37, April 2017.

\bibitem{feng1996spectra}
Keqin Feng and Wen-Ching~Winnie Li.
\newblock Spectra of hypergraphs and applications.
\newblock {\em Journal of number theory}, 60(1):1--22, 1996.

\bibitem{gallagher2013clustering}
Suzanne~Renick Gallagher and Debra~S Goldberg.
\newblock Clustering coefficients in protein interaction hypernetworks.
\newblock In {\em Proceedings of the International Conference on
  Bioinformatics, Computational Biology and Biomedical Informatics}, page 552.
  ACM, 2013.

\bibitem{ghoshdastidar2014consistency}
Debarghya Ghoshdastidar and Ambedkar Dukkipati.
\newblock Consistency of spectral partitioning of uniform hypergraphs under
  planted partition model.
\newblock In {\em Advances in Neural Information Processing Systems}, pages
  397--405, 2014.

\bibitem{ghoshdastidar2015provable}
Debarghya Ghoshdastidar and Ambedkar Dukkipati.
\newblock A provable generalized tensor spectral method for uniform hypergraph
  partitioning.
\newblock In {\em International Conference on Machine Learning}, pages
  400--409, 2015.

\bibitem{ghoshdastidar2015spectral}
Debarghya Ghoshdastidar and Ambedkar Dukkipati.
\newblock Spectral clustering using multilinear svd: Analysis, approximations
  and applications.
\newblock In {\em AAAI}, pages 2610--2616, 2015.

\bibitem{ghoshdastidar2017consistency}
Debarghya Ghoshdastidar and Ambedkar Dukkipati.
\newblock Consistency of spectral hypergraph partitioning under planted
  partition model.
\newblock {\em The Annals of Statistics}, 45(1):289--315, 2017.

\bibitem{golub1996matrix}
Gene~H. Golub and Charles~F. Van~Loan.
\newblock {\em Matrix Computations (3rd Ed.)}.
\newblock Johns Hopkins University Press, Baltimore, MD, USA, 1996.

\bibitem{gu2015subspace}
Ming Gu.
\newblock Subspace iteration randomization and singular value problems.
\newblock {\em SIAM Journal on Scientific Computing}, 37(3):A1139--A1173, 2015.

\bibitem{han1998hypergraph}
Eui-Hong Han, George Karypis, Vipin Kumar, and Bamshad Mobasher.
\newblock Hypergraph based clustering in high-dimensional data sets: A summary
  of results.
\newblock {\em IEEE Data Eng. Bull.}, 21(1):15--22, 1998.

\bibitem{hein2013total}
Matthias Hein, Simon Setzer, Leonardo Jost, and Syama~Sundar Rangapuram.
\newblock The total variation on hypergraphs-learning on hypergraphs revisited.
\newblock In {\em Advances in Neural Information Processing Systems}, pages
  2427--2435, 2013.

\bibitem{jalali2015relative}
Amin Jalali, Qiyang Han, Ioana Dumitriu, and Maryam Fazel.
\newblock Relative density and exact recovery in heterogeneous stochastic block
  models.
\newblock {\em arXiv preprint arXiv:1512.04937}, 2015.

\bibitem{jalali2016exploiting}
Amin Jalali, Qiyang Han, Ioana Dumitriu, and Maryam Fazel.
\newblock Exploiting tradeoffs for exact recovery in heterogeneous stochastic
  block models.
\newblock In {\em Advances in Neural Information Processing Systems}, pages
  4871--4879, 2016.

\bibitem{jerrum1992large}
Mark Jerrum.
\newblock Large cliques elude the metropolis process.
\newblock {\em Random Structures \& Algorithms}, 3(4):347--359, 1992.

\bibitem{kim2018stochastic}
Chiheon Kim, Afonso~S Bandeira, and Michel~X Goemans.
\newblock Stochastic block model for hypergraphs: Statistical limits and a
  semidefinite programming approach.
\newblock {\em arXiv preprint arXiv:1807.02884}, 2018.

\bibitem{kim2011higher}
Sungwoong Kim, Sebastian Nowozin, Pushmeet Kohli, and Chang~D Yoo.
\newblock Higher-order correlation clustering for image segmentation.
\newblock In {\em Advances in neural information processing systems}, pages
  1530--1538, 2011.

\bibitem{kolda2009tensor}
Tamara~G Kolda and Brett~W Bader.
\newblock Tensor decompositions and applications.
\newblock {\em SIAM review}, 51(3):455--500, 2009.

\bibitem{kuvcera1995expected}
Lud{\v{e}}k Ku{\v{c}}era.
\newblock Expected complexity of graph partitioning problems.
\newblock {\em Discrete Applied Mathematics}, 57(2-3):193--212, 1995.

\bibitem{leordeanu2012efficient}
Marius Leordeanu and Cristian Sminchisescu.
\newblock Efficient hypergraph clustering.
\newblock In {\em Artificial Intelligence and Statistics}, pages 676--684,
  2012.

\bibitem{levin1986average}
Leonid~A Levin.
\newblock Average case complete problems.
\newblock {\em SIAM J. Comput.}, 15(1):285--286, feb 1986.

\bibitem{lin2017fundamental}
Chung-Yi Lin, I~Eli Chien, and I-Hsiang Wang.
\newblock On the fundamental statistical limit of community detection in random
  hypergraphs.
\newblock In {\em Information Theory (ISIT), 2017 IEEE International Symposium
  on}, pages 2178--2182. IEEE, 2017.

\bibitem{lu2012loose}
Linyuan Lu and Xing Peng.
\newblock Loose laplacian spectra of random hypergraphs.
\newblock {\em Random Structures \& Algorithms}, 41(4):521--545, 2012.

\bibitem{massoulie2014community}
Laurent Massouli{\'e}.
\newblock Community detection thresholds and the weak ramanujan property.
\newblock In {\em Proceedings of the forty-sixth annual ACM symposium on Theory
  of computing}, pages 694--703. ACM, 2014.

\bibitem{mossel2015reconstruction}
Elchanan Mossel, Joe Neeman, and Allan Sly.
\newblock Reconstruction and estimation in the planted partition model.
\newblock {\em Probability Theory and Related Fields}, 162(3-4):431--461, 2015.

\bibitem{mossel2018proof}
Elchanan Mossel, Joe Neeman, and Allan Sly.
\newblock A proof of the block model threshold conjecture.
\newblock {\em Combinatorica}, 38(3):665--708, 2018.

\bibitem{mossel2016consistency}
Elchanan Mossel, Joe Neeman, Allan Sly, et~al.
\newblock Consistency thresholds for the planted bisection model.
\newblock {\em Electronic Journal of Probability}, 21, 2016.

\bibitem{oymak2011finding}
Samet Oymak and Babak Hassibi.
\newblock Finding dense clusters via low rank+ sparse decomposition.
\newblock {\em arXiv preprint arXiv:1104.5186}, 2011.

\bibitem{pal2019community}
Soumik Pal and Yizhe Zhu.
\newblock Community detection in the sparse hypergraph stochastic block model.
\newblock {\em arXiv preprint arXiv:1904.05981}, 2019.

\bibitem{gonzalez2007hypergraph}
David~A Papa and Igor~L Markov.
\newblock Hypergraph partitioning and clustering.
\newblock In {\em Handbook of Approximation Algorithms and Metaheuristics},
  pages 959--978. Chapman and Hall/CRC, 2007.

\bibitem{vazquez2009finding}
Alexei Vazquez.
\newblock Finding hypergraph communities: a bayesian approach and variational
  solution.
\newblock {\em Journal of Statistical Mechanics: Theory and Experiment},
  2009(07):P07006, 2009.

\bibitem{vershynin2018high}
Roman Vershynin.
\newblock {\em High-dimensional probability: An introduction with applications
  in data science}, volume~47.
\newblock Cambridge University Press, 2018.

\bibitem{zhang2016minimax}
Anderson~Y Zhang and Harrison~H Zhou.
\newblock Minimax rates of community detection in stochastic block models.
\newblock {\em The Annals of Statistics}, 44(5):2252--2280, 2016.

\bibitem{zhou2007learning}
Denny Zhou, Jiayuan Huang, and Bernhard Sch{\"o}lkopf.
\newblock Learning with hypergraphs: Clustering, classification, and embedding.
\newblock In {\em Advances in neural information processing systems}, pages
  1601--1608, 2007.

\end{thebibliography}

\end{document}